\documentclass[letterpaper]{article} 
\usepackage{aaai2026}  

\usepackage{amsmath}
\usepackage{booktabs}
\usepackage{times}  
\usepackage{helvet}  
\usepackage{courier}  
\usepackage[hyphens]{url}  
\usepackage{siunitx}     
\usepackage{multirow}    
\usepackage{graphicx} 
\usepackage{natbib}  
\usepackage{caption} 
\usepackage{algorithm}
\usepackage{algorithmic}
\usepackage{amssymb}
\usepackage{amsthm}
\usepackage{newfloat}
\usepackage{listings}
\DeclareCaptionStyle{ruled}{labelfont=normalfont,labelsep=colon,strut=off} 
\floatstyle{ruled}
\newfloat{listing}{tb}{lst}{}
\floatname{listing}{Listing}
\title{A General Anchor-Based Framework for Scalable Fair Clustering}

\author{
    Shengfei Wei\equalcontrib\textsuperscript{\rm 1},
    Suyuan Liu\equalcontrib\textsuperscript{\rm 1},
    Jun Wang\textsuperscript{\rm 1},
    Ke Liang\textsuperscript{\rm 1},
    Miaomiao Li\textsuperscript{\textdagger \rm 2}\\,
    Lei Luo\thanks{Corresponding authors}\textsuperscript{\rm 1}\\ 
}

\affiliations{
    \textsuperscript{\rm 1} College of Computer Science and Technology, National University of Defense Technology, Changsha, 410073, China\\
    \textsuperscript{\rm 2} School of Computer, Changsha College, Changsha, 410022, China\\
    weishengfei@nudt.edu.cn,
    suyuanliu@nudt.edu.cn, 
    wang\_jun@nudt.edu.cn,
    miaomiaolinudt@gmail.com,
    l.luo@nudt.edu.cn
}

\begin{document}

\maketitle

\begin{abstract}
Fair clustering is crucial for mitigating bias in unsupervised learning, yet existing algorithms often suffer from quadratic or super-quadratic computational complexity, rendering them impractical for large-scale datasets. To bridge this gap, we introduce the Anchor-based Fair Clustering Framework (AFCF), a novel, general, and plug-and-play framework that empowers arbitrary fair clustering algorithms with linear-time scalability. Our approach first selects a small but representative set of anchors using a novel fair sampling strategy. Then, any off-the-shelf fair clustering algorithm can be applied to this small anchor set. The core of our framework lies in a novel anchor graph construction module, where we formulate an optimization problem to propagate labels while preserving fairness. This is achieved through a carefully designed group-label joint constraint, which we prove theoretically ensures that the fairness of the final clustering on the entire dataset matches that of the anchor clustering. We solve this optimization efficiently using an ADMM-based algorithm. Extensive experiments on multiple large-scale benchmarks demonstrate that AFCF drastically accelerates state-of-the-art methods, which reduces computational time by orders of magnitude while maintaining strong clustering performance and fairness guarantees.

\end{abstract}

\begin{links}
    \link{Code}{https://github.com/smcsurvey/AFCF}
    \link{Extended version}{https://aaai.org/example/extended-version}
\end{links}

\section{Introduction}
Machine learning has been widely applied in key domains such as finance, education, and healthcare. It is concerning that models\cite{zhou25,wang2024view,AIRMVC} may exhibit discrimination against groups characterized by sensitive attributes such as race and gender, due to biases inherent in the data\cite{chouldechova2020snapshot,buolamwini2018gender}. Recently, such issues have sparked extensive interest within the community regarding algorithmic fairness in supervised learning\cite{zafar2017fairness,donini2018empirical,huang2022fair} Moreover, the incorporation of algorithmic fairness constraints has begun to be explored in unsupervised learning\cite{chierichetti2017fair,kleindessner2019guarantees,li2024one}. Specifically, Chierichetti et al.\cite{chierichetti2017fair} pioneered the concept of fair clustering, advocating for approximate proportional representation of samples from each sensitive group within every cluster.

Recent research has focused on develoxping methods to ensure group fairness in clustering models. Several works have specifically addressed fairness in prototype-based clustering, e.g., \cite{chierichetti2017fair,carreira2013k,backurs2019scalable,bera2019fair}. Very recently, Kleindessner et al.\cite{kleindessner2019guarantees} formalized this notion of group fairness within the spectral clustering framework\cite{shi2000normalized,von2007tutorial}. However, constrained by computational complexity and other factors, many existing approaches lack scalability. The pioneering fair clustering approach introduced by Chierichetti et al.\cite{chierichetti2017fair} exhibits super-quadratic runtime complexity, primarily due to its initial fairlet decomposition phase, resulting in significant scalability limitations\cite{chhabra2021overview}. Similarly, Kleindessner et al.\cite{kleindessner2019guarantees} incorporated linear fairness constraints on the assignment matrix within a spectral relaxation framework. However, their method necessitates storing the full affinity matrix and computing its eigenvalue decomposition. This incurs cubic complexity for a direct implementation, and while optimized approximations exist, they typically achieve only super-quadratic complexity\cite{tian2014learning}, thereby imposing substantial scalability constraints.

Despite progress in enhancing the scalability of fair clustering, significant limitations persist. For instance, Backurs et al.\cite{backurs2019scalable} introduced a tree-based metric approach to construct fairlets in near-linear time. However, this method is restricted to settings with only two protected groups. Although Wang et al.\cite{wang2023scalable} accelerated computations for fair spectral clustering (SC), the persistent reliance on the computationally expensive eigendecomposition of the fairness-constrained graph Laplacian remains a fundamental bottleneck, limiting its practical adoption. Concurrently, Chhabra et al.\cite{chhabra2022robust} proposed a more general framework for fair clustering using antidote data. However, the authors explicitly note its primary limitation: prohibitively expensive computational costs when applied to high-dimensional or large-scale datasets.

To address these issues, inspired by prototype learning\cite{Yalanq254}, we propose a general anchor-based fair clustering framework (Figure. \ref{fig:framwork}) consists of four modules: Fair Anchor Generation Module, Anchor Fair Clustering Module, Fair Anchor Graph Construction Module, Label Propagation Module. This method introduces anchors into fair clustering to reduce the clustering scope from $n$ to $m$, 
where $n$ and $m$ denote the dataset size and number of anchors respectively ($m \ll n$). Specifically, we first select demographically balanced anchors via quota constraints. 
We then apply any fair clustering algorithm to these anchors to generate cluster labels 
(supporting plug-and-play algorithm substitution). 
Next, we solve the fairness-constrained anchor graph matrix using ADMM optimization. 
Finally, we propagate anchor labels to the entire dataset through label propagation. 

This framework confines high-complexity fair clustering operations to the anchor subset, 
while enabling global propagation of fairness constraints via the anchor graph, 
allowing computationally intensive fair clustering algorithms to scale efficiently to large-scale datasets.The key contributions of our work can be summarized as follows:

\begin{itemize}
    \item To address the pervasive quadratic or super-quadratic complexity limitation in existing fair clustering methods, we propose a novel general anchor-based fair clustering framework. This achieves linear-time scalability for arbitrary fair clustering algorithms by reducing the problem to optimizing over $m$ anchors ($m \ll n$), maintaining clustering quality and fairness while lowering complexity from $\mathcal{O}(n^k)$ ($k \geq 2$) to $\mathcal{O}(n)$, providing a universal solution for large-scale fair clustering.
    
    \item We introduce a protected group-label co-constraint mechanism during anchor graph construction, which establishes theoretical guarantees for fairness equivalence transfer between anchor-based clustering and global clustering. This fairness-preserving module ensures that non-anchor assignments simultaneously satisfy cluster cohesion and demographic parity through ADMM optimization.
    
    \item Extensive experiments with multiple fair clustering algorithms across diverse datasets demonstrate the framework's efficacy. It achieves significant speedup over existing methods while preserving clustering performance and group balance, particularly under increasing data scales.
\end{itemize}

\begin{figure}[!t]
  \centering
  \includegraphics[width=0.9\linewidth]{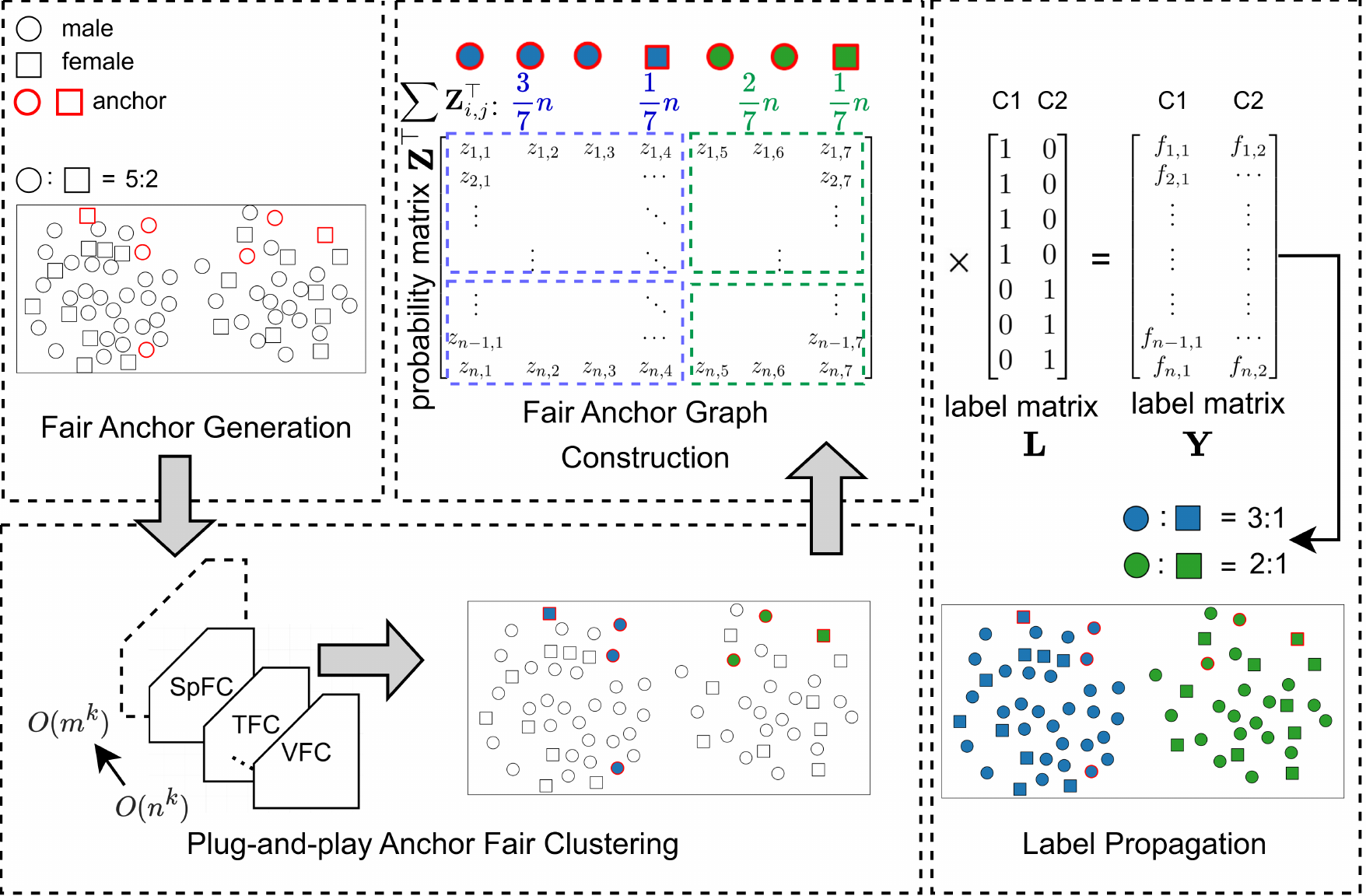}
  \caption{Conceptual framework of the proposed fair anchor-based clustering. 
  Anchors are selected proportionally to cluster demographics, with the left yielding 3 circle and 1 square anchors 
  and the right yielding 2 circle and 1 square anchors. 
  This enables efficient fair clustering on anchor points only, where m $\ll$ n. 
  Group-label joint constraints in the anchor graph $\mathbf{Z}$ maintain demographic proportions 
  with sums of 3/7 for blue circles, 1/7 for blue squares, 2/7 for orange circles, and 1/7 for orange squares. 
  These fairness properties propagate to final clusters through $\mathbf{Y} = \mathbf{Z}^\top\mathbf{L}$, 
  preserving the original demographic ratios. 
  The approach integrates proportional representation, plug-and-play algorithmic flexibility, 
  and constrained graph optimization for fairness preservation.}
  \label{fig:framwork}
\end{figure}

\section{Related Work}
\subsection{Fair Clustering}
In recent years, fair clustering has received growing attention in the machine learning community. Since conventional clustering methods may produce biased outcomes due to the influence of sensitive attributes, significant research efforts have been dedicated to developing fairness-constrained clustering methodologies\cite{ahmadian2019clustering,chhabra2021overview,shaham2025privacy}. Chierichetti et al.\cite{chierichetti2017fair} pioneered the concept of fair clustering for binary protected groups, while Bera et al.\cite{bera2019fair} extended this framework to multidimensional protected groups, formalizing fairness evaluation through cluster balance metrics.Fair clustering methods quantify fairness through balance metrics defined on protected groups. Specifically, consider a dataset $\mathbf{X} \in \mathbb{R}^{d \times n}$ with $n$ samples and $d$ features, partitioned into $c$ disjoint clusters $\mathcal{C} = \{C_1, C_2, \dots, C_c\}$. For $t$ mutually exclusive protected groups $\mathcal{G} = \{G_1, G_2, \dots, G_t\}$, define the global proportion $\rho_r = |G_r|/n$ and intra-cluster proportion $\rho_r^{(l)} = |C_l \cap G_r|/|C_l|$. The cluster fairness metric is formulated as:
\begin{equation}
\label{S2_1}
\text{balance}(\mathcal{C}) = \min_{l \in [c]} \min_{\substack{r, r' \in [t] \\ r \neq r'}} \frac{\rho_r^{(l)}}{\rho_{r'}^{(l)}}
\end{equation}
where $[c] = \{1, 2, \dots, c\}$, $[t] = \{1, 2, \dots, t\}$,with higher values indicating fairer cluster distributions.

Building upon this concept, researchers have proposed diverse approaches to fair clustering from various perspectives. Numerous methods incorporate fairness constraints into the optimization objective. For instance, SpFC\cite{kleindessner2019guarantees} embeds fairness as linear constraints within spectral clustering. However, due to the necessity of storing the $n \times n$ affinity matrix and computing its eigendecomposition—which remains super-quadratic even for fast implementations—the approach\cite{tian2014learning} suffers from scalability limitations. While Wang et al.\cite{wang2023scalable} accelerated computations for fair spectral clustering, their method still relies on the computationally expensive eigendecomposition of the fairness-constrained graph Laplacian. FairSC\cite{tonin2025accelerating} reformulates the fair spectral clustering problem within a difference-of-convex (DC) programming framework, yet maintains quadratic complexity $\mathcal{O}(n^2)$.
VFC\cite{ziko2021variational} incorporates fairness penalties into the clustering objective through a variational framework, achieving scalability on large-scale datasets via decoupled updates of assignment variables.Some approaches employ pre-processing or post-processing techniques. For instance, Chierichetti et al.\cite{chierichetti2017fair} transform raw data into fairlet representations satisfying fairness constraints, enabling fair clustering through classical algorithms. However, this method exhibits quadratic time complexity with respect to the number of data points.Backurs et al.\cite{backurs2019scalable} accelerated fairlet decomposition to near-linear time complexity, though their approach remains restricted to binary protected groups. Concurrently,Chhabra et al.\cite{chhabra2022fair} proposed a more general antidote data fair clustering approach, but the authors explicitly note its computational inefficiency on large-scale datasets as a primary limitation.CFC\cite{chhabra2022robust} achieves fairness through data pre-processing or re-sampling techniques. In contrast, AFC\cite{bera2019fair} transforms clustering outputs into fair solutions via linear programming formulations. Additionally, deep fair clustering methods have emerged, exemplified by DFC\cite{li2020deep} which enforces fairness through adversarial training.
\subsection{Anchor-Based Clustering}
To address the computational challenges of traditional clustering methods\cite{wangevaluate,CausalMVC} on large-scale datasets, recent works have introduced anchor-based techniques \cite{qiang2021fast,nie2023fast,liu2024learn} to accelerate and optimize the clustering process.The core principle of anchor-based clustering is to generate a set of representative anchors from the data pool, where the number of anchors $m$ is typically orders of magnitude smaller than the total number of samples\cite{zhang2023efficient,liu2022fast}. A similarity graph is then constructed to measure affinities between samples and anchors. Subsequently, clustering procedures are performed on the representation matrix. For example, Wang et al.\cite{wang2021fast} proposed a self-supervised clustering model to learn the similarity matrix; Liu et al.\cite{liu2024learn} developed an iteratively optimized anchor strategy for multi-view clustering tasks; Nie et al.\cite{nie2023fast} established a unified framework based on anchors that accelerates clustering by updating compact label matrices. These anchor-based methods achieve speedup without sacrificing performance when appropriate anchors are selected.

\section{Method}
In this section, we present a general anchor-based fair clustering framework for large-scale fair clustering. We first introduce the four modules of this framework in sequence: the Fair Anchor Generation Module, the Anchor Fair Clustering Module, the Fair Anchor Graph Construction Module, and the Label Propagation Module. Finally, we conduct a complexity analysis of the proposed algorithm.
\subsection{Overview}
Existing fair clustering methods face difficulties in scaling to large-scale data due to super-quadratic complexity. To address this, we propose a general Anchor-based Fair Clustering Framework (AFCF) that enables linear-time scalability for arbitrary fair clustering algorithms. First, we employ a novel fair sampling strategy to select a small yet representative anchor set. Then, any off-the-shelf fair clustering algorithm is applied to this anchor subset. Next, through a carefully designed group-label co-constraint, we formulate an optimization problem that ensures fairness during anchor graph construction, with theoretical guarantees that this constraint preserves identical fairness properties between anchor clustering and final full-dataset clustering. Finally, we solve this optimization via ADMM and propagate anchor labels to the entire dataset to obtain the final clustering.
\subsection{Fair Anchor Generation}
Two prevalent anchor selection strategies are random sampling and k-means \cite{jain2010data}. However, the stochastic nature of the former yields unstable results with no performance guarantees, while the latter's sensitivity to initial centroids causes algorithmic instability, often requiring multiple runs to mitigate randomness effects. Crucially, neither approach considers fair representation across protected demographic groups.

Li et al.\cite{li2020multiview} proposed a simple yet effective anchor selection strategy called Directly Alternate Sampling (DAS), which selects anchors that comprehensively cover the entire data point cloud. Building on this coverage principle while ensuring fair representation across protected demographic groups, we propose the Fair Directly Alternate Sampling (FDAS) method. Algorithm \ref{alg:FDAS} summarizes the FDAS procedure. FDAS introduces a dual fairness mechanism on top of DAS: first, proportional quota allocation (Step 2) ensures anchor counts per protected group match global demographic proportions by computing base quotas, and distributes the residual anchors $\Delta$ to the groups with the smallest current counts to avoid certain groups being completely ignored; second, within-group selection employs nonlinear score decay (Step 15) to promote spatial diversity by suppressing scores near selected samples. This design preserves DAS's spatial coverage capability while significantly improving protected group representation.
\begin{algorithm}
\caption{FDAS} 
\label{alg:FDAS} 
\begin{algorithmic}[1] 
\REQUIRE 
    Data matrix $\mathbf{X} \in \mathbb{R}^{d \times n}$, 
    Number of anchors $m$, 
    Group proportion vector $\mathbf{t} \in \mathbb{R}^{t}$,
    Group labels $\text{group} \in \mathbb{Z}^n$
\ENSURE 
    Anchor indices $\mathcal{A}$
    
\STATE Normalize data: $\mathbf{X} \gets \mathbf{X} - \min(\mathbf{X})$
\STATE Compute group quotas: $\text{counts} \gets \lfloor m \cdot \mathbf{t} \rfloor$
\STATE $\Delta \gets m - \sum_{g=0}^{t-1} \text{counts}[g]$
\FOR{$k=1$ \textbf{to} $\Delta$}
    \STATE $g^* \gets \arg\min_g \text{counts}[g]$
    \STATE $\text{counts}[g^*] \gets \text{counts}[g^*] + 1$
\ENDFOR

\FOR{$g=0$ \textbf{to} $t - 1$}
    \STATE  $\mathcal{G} \gets \{i \mid \text{group}[i] = g\}$
    \STATE  $\mathbf{s} \gets \left[ \sum_{p=1}^{d} \mathbf{X}_{p,i} \right]_{i \in \mathcal{G}}$
    \STATE  $\mathbf{s} \gets \mathbf{s} / \max(\mathbf{s})$
    
    \FOR{$j=1$ \textbf{to} $\text{counts}[g]$}
        \STATE $k \gets \arg\max \mathbf{s}$
        \STATE $\mathcal{A} \gets \mathcal{A} \cup \{\mathcal{G}[k]\}$
        \STATE$\mathbf{s} \gets \mathbf{s} \odot (1 - \mathbf{s}) / \max(\mathbf{s})$ \COMMENT{$\odot$ denotes element-wise multiplication}
    \ENDFOR
\ENDFOR

\STATE \textbf{return} $\mathcal{A}$
\end{algorithmic} 
\end{algorithm}

\subsection{Anchor Fair Clustering Module}
In the Anchor Fair Clustering module, we apply an arbitrary fair clustering operator to the anchors obtained from FDAS:
\begin{equation}
\mathbf{l} = \mathcal{F}(\mathcal{A}, \mathbf{g}_{\mathcal{A}}, k)
\end{equation}
where 
$\mathcal{A}$ denotes the $m \times d$ anchor matrix, 
$\mathbf{g}_{\mathcal{A}}$ represents the sensitive attribute vector for anchors, 
$k$ specifies the number of clusters, 
$\mathcal{F}$ corresponds to the fair clustering operator interface, and 
$\mathbf{l}$ is the $m \times 1$ cluster assignment vector. 
Given $m \ll n$, this module enables seamless substitution with various fair clustering algorithms $\mathcal{F} \in \{\textsc{SpFC}, \textsc{VFC}, \cdots\}$, effectively circumventing scalability bottlenecks in large-scale clustering scenarios.
\subsection{Fair Anchor Graph Construction Module}
\subsubsection{Problem Formulation}

To ensure lossless transfer of fairness from anchor clustering to the final partitioning, we propose a provable fairness maintenance mechanism (Proposition~\ref{prop:fairness-preserving}). Unlike traditional anchor graph methods that solely optimize reconstruction accuracy, our approach introduces intersectional group-label constraints that preserve demographic parity across sensitive attributes and cluster assignments. 

Formally, given data matrix $\mathbf{X} \in \mathbb{R}^{d \times n}$ and anchor features $\mathbf{H} \in \mathbb{R}^{d \times m}$, we construct the anchor graph $\mathbf{Z} \in \mathbb{R}^{m \times n}$ by solving:
\begin{equation}\label{eq:optimization}
\min_{\mathbf{Z}} \|\mathbf{X} - \mathbf{H} \mathbf{Z}\|_F^2 + \alpha \|\mathbf{Z}\|_F^2
\end{equation}
\centerline{$s.t. \;\; \sum\limits_{j \in \mathcal{C}_l} \sum\limits_{i \in G_r} \mathbf{Z}_{j,i} = t_{l,r}, \quad \forall l \in [c],  r \in [t],$}
\centerline{$\mathbf{Z}_{:,i} \in \Delta^m, \quad \forall i$}
where $\Delta^m := \{ z \in \mathbb{R}^m \mid z_j \geq 0, \sum_j z_j = 1 \}$ denotes the standard simplex, $\mathcal{G}_{l,r}$ denotes anchor subgroups defined by the conjunction of sensitive attribute $r$ and cluster label $l$, and $t_{l,r} = |\mathcal{G}_{l,r}| \cdot n/m$. This constraint enforces proportional representation consistency, as formally stated in Proposition~\ref{prop:fairness-preserving}, whose proof is provided in the appendix.
\newtheorem{proposition}{Proposition}
\begin{proposition}\label{prop:fairness-preserving}
Under the fairness constraints in (\ref{eq:optimization}), the balance metric $\text{balance}(\mathcal{C})$ of the final clustering equals that of the anchor clustering $\text{balance}(\mathcal{C}_{\text{a}})$:
\begin{equation}\label{eq:balance-equal}
\text{balance}(\mathcal{C}) = \text{balance}(\mathcal{C}_{\text{a}})
\end{equation}
\end{proposition}

\subsubsection{ADMM Optimization}
To efficiently solve the constrained optimization problem in Eq.~\eqref{eq:optimization}, we introduce an auxiliary variable $\mathbf{E}$ with equality constraint $\mathbf{Z} = \mathbf{E}$. This strategic decoupling separates the simplex constraints from the fairness conditions, isolates the fairness constraints to the $\mathbf{E}$-subproblem while maintaining the probability simplex constraints on $\mathbf{Z}$, connected through the consensus constraint $\mathbf{Z} = \mathbf{E}$ \cite{boyd2011distributed}.

We adopt the Alternating Direction Method of Multipliers (ADMM) to solve. Specifically, the augmented Lagrangian is:

\begin{equation}
\label{eq:augmented_lag}
\begin{split}
\mathcal{L}_\rho(\mathbf{Z}, \mathbf{E}, \mathbf{\Lambda}) = & \|\mathbf{X} - \mathbf{H} \mathbf{Z}\|_F^2 + \alpha \|\mathbf{Z}\|_F^2 + \langle \mathbf{\Lambda}, \mathbf{Z} - \mathbf{E} \rangle \\
& + \frac{\rho}{2} \|\mathbf{Z} - \mathbf{E}\|_F^2
\end{split}
\end{equation}
where $\mathbf{\Lambda} \in \mathbb{R}^{m \times n}$ is the dual variable and $\rho > 0$ is the penalty parameter. The ADMM iterations proceed as:

\begin{align}
\mathbf{Z}^{k+1} &= \underset{\mathbf{Z} \in \Delta^m}{argmin} \ \mathcal{L}_\rho(\mathbf{Z}, \mathbf{E}^k, \mathbf{\Lambda}^k) \label{eq:z_update} \\
\mathbf{E}^{k+1} &= \underset{E}{argmin} \ \mathcal{L}_\rho(\mathbf{Z}^{k+1}, \mathbf{E}, \mathbf{\Lambda}^k) \label{eq:e_update} \\
\mathbf{\Lambda}^{k+1} &= \mathbf{\Lambda}^k + \rho (\mathbf{Z}^{k+1} - \mathbf{E}^{k+1}) \label{eq:lambda_update}
\end{align}
Our algorithm is summarized in Algorithm~\ref{alg:admm_concise}. The penalty $\rho^{(i+1)}$ is updated according to the standard rule suggested in~\cite{boyd2011distributed} and detailed in Appendix. We now analyze the two subproblems separately.
\subsubsection{Solution to $\mathbf{Z}$-subproblem}
The $\mathbf{Z}$-update in Eq.~\eqref{eq:z_update} decomposes into $n$ independent subproblems along data points:

\begin{equation}
\label{eq:z_sub}
\min_{z_i \in \Delta^m} \ \frac{1}{2} z_i^\top \mathbf{Q} z_i + c_i^\top z_i
\end{equation}
with $\mathbf{Q} = 2(\mathbf{H}^\top \mathbf{H} + (\alpha + \rho/2)\mathbf{I})$ and $c_i = -2\mathbf{H}^\top x_i - \rho e_i^k + \lambda_i^k$. We solve this quadratic program over the simplex using the Frank-Wolfe algorithm \cite{nanculef2014novel}, selected for its projection-free property and linear convergence rate in convex problems over polytopes. The complete procedure is detailed in Appendix.

\subsubsection{Solution to $\mathbf{E}$-subproblem}
The $\mathbf{E}$-update in Eq.~\eqref{eq:e_update} requires projection onto the fairness constraints. Defining $\mathbf{R} = \mathbf{Z}^{k+1} + \rho^{-1}\mathbf{\Lambda}^k$, we obtain the closed-form solution:

\begin{equation}
\label{eq:e_closed}
\mathbf{E}_{j,i} = \mathbf{R}_{j,i} + \frac{t_{l,r} - \sum_{a \in \mathcal{C}_l} \sum_{b \in G_r} \mathbf{R}_{a,b}}{|\mathcal{C}_l| \cdot |G_r|}, \quad \forall j \in \mathcal{C}_l, i \in G_r
\end{equation}

where $\mathcal{C}_l$ denotes the set of anchors assigned to cluster $l$, and $G_r$ denotes the set of samples belonging to sensitive group $r$. This operation redistributes mass uniformly within each group $\mathcal{C}_l \times G_r$ to satisfy the fairness constraint $\sum_{j \in \mathcal{C}_l} \sum_{i \in G_r} \mathbf{E}_{j,i} = t_{l,r}$.

\begin{algorithm}[t]
\caption{Fair Anchor Graph Construction}
\label{alg:admm_concise}
\begin{algorithmic}[1] 
\REQUIRE 
    Data matrix $\mathbf{X} \in \mathbb{R}^{d \times n}$, 
    Anchor features $\mathbf{H} \in \mathbb{R}^{d \times m}$, 
    Group definitions $\{\mathcal{G}_{l,r}\}$, 
    Regularization parameter $\alpha > 0$, 
    Initial penalty $\rho_0 > 0$, 
    Tolerance $\epsilon > 0$, 
    Maximum iterations $K$
\ENSURE 
    Fair anchor graph $\mathbf{Z} \in \mathbb{R}^{m \times n}$
    
\STATE Initialize $\mathbf{Z}^0 \gets \mathbf{1}_m\mathbf{1}_n^\top/m$ 
\STATE Initialize $\mathbf{E}^0 \gets \mathbf{Z}^0$, $\mathbf{\Lambda}^0 \gets \mathbf{0}_{m \times n}$, $k \gets 0$
\REPEAT
    \STATE Update $\mathbf{Z}^{k+1}$ according to (\ref{eq:z_update}) 
    \STATE Update $\mathbf{E}^{k+1}$ according to (\ref{eq:e_update})
    \STATE Update $\mathbf{\Lambda}^{k+1}$ according to (\ref{eq:lambda_update}) 
    \IF{$k \mod 10 = 0$} 
        \STATE Update $\rho$ \COMMENT{details in Appendix}
    \ENDIF
    \STATE Compute $r_k \gets \|\mathbf{Z}^{k+1} - \mathbf{E}^{k+1}\|_F$
    \STATE Compute $s_k \gets \rho \|\mathbf{E}^{k+1} - \mathbf{E}^{k}\|_F$
    \STATE $k \gets k + 1$
\UNTIL{$k > K$ \textbf{or} $\max(r_k, s_k) < \epsilon$}
\STATE \textbf{return} $\mathbf{Z}^k$
\end{algorithmic}
\end{algorithm}

\subsection{Label Propagation Module}
Since the anchor graph $\mathbf{Z}$ satisfies group ratio consistency through the group-label joint constraint in Eq.~\eqref{eq:optimization}, we propagate anchor labels to the entire dataset via single-step matrix multiplication.This design inherently preserves the fairness property, which grounded in the principle that spatially proximate samples in the feature space exhibit higher likelihood of sharing identical class labels \cite{xu2019label}. Formally, based on the anchor fair clustering results $\mathbf{l}$, we construct the anchor label matrix $\mathbf{L} \in \{0,1\}^{m \times k}$. The $\mathbf{L}$ provides one-hot encoded cluster assignments, where $\mathbf{L}_{i,j}=1$ iff anchor $j$ belongs to cluster $i$.
Subsequently, label diffusion is achieved through matrix multiplication:
\begin{equation}
\mathbf{Y} = \mathbf{Z}^\top \mathbf{L}
\end{equation}
where $\mathbf{Y} \in \mathbb{R}^{n \times k}$ denotes the probabilistic label assignment matrix, with element $\mathbf{Y}_{ij}$ quantifying the membership likelihood of sample $i$ belonging to class $j$ \cite{xu2019label}. This operation embodies the neighborhood consensus principle in fair clustering: each sample's label distribution emerges as a convex combination of its associated anchors' labels, weighted by their affinity measures $\mathbf{Z}$ \cite{raghavan2007near}. The discrete cluster assignments $\hat{y}_i$ are then determined via maximum likelihood decision:
\begin{equation}
\hat{y}_i = argmax_{1 \leq j \leq k} \mathbf{Y}_{ij}, \quad \forall i \in \{1,\dots,n\}
\end{equation}
This module achieves end-to-end label propagation with linear time complexity, eliminating requirements for iterative optimization or post-processing while preserving \textit{fairness-awareness} through geometrically consistent label diffusion.

\subsection{Complexity Analysis}
The AFCF framework achieves linear scalability in sample size $n$ with an overall complexity of $O(nmd + nm^2 + nmk)$. 
This derives from four main components: 
fair anchor generation requires $O(nd + mn)$ operations for data processing and selection; 
anchor clustering incurs $O(f(m))$ complexity where $f$ is the embedded algorithm's complexity; 
fair anchor graph construction costs $O(nm^2 + nmd)$; 
and label propagation takes $O(nmk)$. 

\begin{table}[t]
\caption{Datasets used in our experiments.}
\label{tab:dataset}
\centering
\small
\begin{tabular}{lcccc}
\toprule
\textbf{Dataset} & \textbf{Samples} & \textbf{Clusters} & \textbf{Protected Groups} \\
\midrule
Law School & 18,692 & 2 & Gender (2) \\
Credit & 29,537 & 5 & Gender (2) \\
Bank & 41,108 & 2 & Marital status (3) \\
Zafar & 100,000 & 2 & Binary (2) \\
Census II & 2,458,285 & 5 & Gender (2) \\
\bottomrule
\end{tabular}
\end{table}

\begin{table*}[t]
\centering
\caption{Comprehensive Method Comparison with AFCF Enhancement. NAN indicates failure to complete within 30 minutes; "-" indicates the method does not support scenarios where the number of sensitive attribute groups exceeds 2.}
\label{tab:full_results}
\scriptsize
\setlength{\tabcolsep}{2pt}
\renewcommand{\arraystretch}{1.0}
\begin{tabular}{@{}l*{20}{S[table-format=1.3]}@{}}
\toprule
\multirow{2}{*}{Method} & 
\multicolumn{4}{c}{Law School} & 
\multicolumn{4}{c}{Credit} & 
\multicolumn{4}{c}{Bank} & 
\multicolumn{4}{c}{Zafar} & 
\multicolumn{4}{c}{Census II} \\
\cmidrule(lr){2-5} \cmidrule(lr){6-9} \cmidrule(lr){10-13} \cmidrule(lr){14-17} \cmidrule(lr){18-21}
& {ACC} & {NMI} & {Balance} & {MNCE} &
{ACC} & {NMI} & {Balance} & {MNCE} &
{ACC} & {NMI} & {Balance} & {MNCE} &
{ACC} & {NMI} & {Balance} & {MNCE} &
{ACC} & {NMI} & {Balance} & {MNCE} \\
\midrule
VFC        & 0.588 & 0.063 & 0.762 & 0.999 & 0.340 & 0.150 & 0.619 & 0.992 & 0.642 & 0.044 & 0.179 & 0.972 & 0.961 & 0.764 & 0.652 & 0.992 & 0.407 & 0.243 & 0.678 & 0.974 \\
VFC-AF     & \textbf{0.693} & \textbf{0.065} & 0.718 & 0.992 & \textbf{0.473} & 0.124 & 0.562 & 0.974 & \textbf{0.719} & \textbf{0.062} & \textbf{0.184} & \textbf{0.976} & \textbf{0.980} & \textbf{0.865} & \textbf{0.655} & \textbf{0.993} & 0.329 & 0.093 & 0.580 & 0.949 \\
\midrule
SpFC       & 0.852 & 0.005 & 0.765 & 0.999 & 0.419 & 0.155 & 0.570 & 0.977 & {NAN} & {NAN} & {NAN} & {NAN} & {NAN} & {NAN} & {NAN} & {NAN} & {NAN} & {NAN} & {NAN} & {NAN} \\
SpFC-AF    & 0.703 & \textbf{0.074} & \textbf{0.765} & \textbf{0.999} & \textbf{0.419} & 0.112 & \textbf{0.573} & \textbf{0.978} & \textbf{0.876} & \textbf{0.143} & \textbf{0.185} & \textbf{0.993} & \textbf{0.993} & \textbf{0.943} & \textbf{0.653} & \textbf{0.993} & \textbf{0.456} & \textbf{0.131} & \textbf{0.906} & \textbf{0.999} \\
\midrule
FMSC       & 0.558 & 0.060 & 0.772 & 1.000 & {NAN} & {NAN} & {NAN} & {NAN} & {NAN} & {NAN} & {NAN} & {NAN} & {NAN} & {NAN} & {NAN} & {NAN} & {NAN} & {NAN} & {NAN} & {NAN} \\
FMSC-AF    & \textbf{0.651} & \textbf{0.067} & 0.738 & 0.996 & \textbf{0.466} & \textbf{0.016} & \textbf{0.640} & \textbf{0.997} & \textbf{0.852} & \textbf{0.148} & \textbf{0.185} & \textbf{0.996} & \textbf{0.753} & \textbf{0.193} & \textbf{0.666} & \textbf{0.996} & \textbf{0.436} & \textbf{0.037} & \textbf{0.835} & \textbf{0.995} \\
\midrule
TFC        & {NAN} & {NAN} & {NAN} & {NAN} & {NAN} & {NAN} & {NAN} & {NAN} & {-} & {-} & {-} & {-} & {NAN} & {NAN} & {NAN} & {NAN} & {NAN} & {NAN} & {NAN} & {NAN} \\
TFC-AF     & \textbf{0.786} & \textbf{0.083} & \textbf{0.726} & \textbf{0.994} & \textbf{0.350} & \textbf{0.111} & \textbf{0.542} & \textbf{0.967} & {-} & {-} & {-} & {-} & \textbf{0.533} & \textbf{0.030} & \textbf{0.687} & \textbf{1.000} & \textbf{0.380} & \textbf{0.139} & \textbf{0.509} & \textbf{0.923} \\
\midrule
FFC        & 0.581 & 0.060 & 0.758 & 0.998 & {NAN} & {NAN} & {NAN} & {NAN} & {NAN} & {NAN} & {NAN} & {NAN} & {NAN} & {NAN} & {NAN} & {NAN} & {NAN} & {NAN} & {NAN} & {NAN} \\
FFC-AF     & \textbf{0.701} & \textbf{0.070} & 0.753 & \textbf{0.998} & \textbf{0.359} & \textbf{0.062} & \textbf{0.584} & \textbf{0.982} & \textbf{0.719} & \textbf{0.062} & \textbf{0.184} & \textbf{0.976} & \textbf{0.617} & \textbf{0.044} & \textbf{0.685} & \textbf{1.000} & \textbf{0.385} & \textbf{0.104} & \textbf{0.393} & \textbf{0.859} \\
\bottomrule
\end{tabular}
\end{table*}

\begin{table*}[t]
\centering
\caption{Ablation Study on Anchor Selection Methods. Random: randomly generated anchors; DAS: anchors selected by DAS method; FDAS: anchors selected by our FDAS module.}
\label{tab:ablation_results}
\scriptsize
\setlength{\tabcolsep}{2pt}
\renewcommand{\arraystretch}{1.0}
\begin{tabular}{@{}l*{20}{S[table-format=1.3]}@{}}
\toprule
\multirow{2}{*}{Method} & 
\multicolumn{4}{c}{Law School} & 
\multicolumn{4}{c}{Credit} & 
\multicolumn{4}{c}{Bank} & 
\multicolumn{4}{c}{Zafar} & 
\multicolumn{4}{c}{Census II} \\
\cmidrule(lr){2-5} \cmidrule(lr){6-9} \cmidrule(lr){10-13} \cmidrule(lr){14-17} \cmidrule(lr){18-21}
& {ACC} & {NMI} & {Balance} & {MNCE} &
{ACC} & {NMI} & {Balance} & {MNCE} &
{ACC} & {NMI} & {Balance} & {MNCE} &
{ACC} & {NMI} & {Balance} & {MNCE} &
{ACC} & {NMI} & {Balance} & {MNCE} \\
\midrule
Random & 0.902 & 0.000 & 0.772 & 1.000 & 0.350 & 0.076 & 0.512 & 0.954 & 0.716 & 0.042 & 0.183 & 0.977 & 0.978 & 0.859 & 0.653 & 0.993 & 0.453 & 0.184 & 0.662 & 0.971 \\
DAS    & 0.902 & 0.000 & 0.772 & 1.000 & 0.440 & 0.105 & 0.560 & 0.973 & 0.721 & 0.061 & 0.185 & 0.977 & 0.936 & 0.712 & 0.650 & 0.992 & 0.388 & 0.045 & 0.101 & 0.442 \\
FDAS   & 0.703 & \textbf{0.074} & 0.765 & 0.999 & 0.419 & \textbf{0.112} & \textbf{0.573} & \textbf{0.978} & \textbf{0.876} & \textbf{0.143} & \textbf{0.185} & \textbf{0.993} & \textbf{0.993} & \textbf{0.943} & \textbf{0.653} & \textbf{0.993} & \textbf{0.456} & 0.131 & \textbf{0.906} & \textbf{0.999} \\
\bottomrule
\end{tabular}
\end{table*}

\begin{table}[t]
\centering
\caption{Execution Time Comparison with AFCF Enhancement (seconds). NAN indicates failure to complete within 30 minutes; "-" indicates the method does not support the scenario.}
\label{tab:time_results}
\scriptsize
\setlength{\tabcolsep}{2pt}
\renewcommand{\arraystretch}{1.0}
\begin{tabular}{@{}l*{5}{S[table-format=4.3]}@{}}
\toprule
\multirow{2}{*}{Method} & 
\multicolumn{5}{c}{Datasets} \\
\cmidrule(lr){2-6}
& {Law School} & {Credit} & {Bank} & {Zafar} & {Census II} \\
\midrule
VFC        & 49.316 & 66.932 & 50.992 & 65.135 & 1494.382 \\
VFC-AF     & \textbf{21.010} & \textbf{23.449} & \textbf{47.206} & \textbf{46.429} & \textbf{918.220} \\
\midrule
SpFC       & 79.367 & 324.095 & {NAN} & {NAN} & {NAN} \\
SpFC-AF    & \textbf{8.907} & \textbf{14.911} & \textbf{34.967} & \textbf{66.139} & \textbf{731.359} \\
\midrule
FMSC       & 179.795 & {NAN} & {NAN} & {NAN} & {NAN} \\
FMSC-AF    & \textbf{13.206} & \textbf{9.777} & \textbf{12.675} & \textbf{40.763} & \textbf{713.174} \\
\midrule
TFC        & {NAN} & {NAN} & {-} & {NAN} & {NAN} \\
TFC-AF     & \textbf{12.002} & \textbf{15.360} & {-} & \textbf{47.512} & \textbf{1022.102} \\
\midrule
FFC        & 393.130 & {NAN} & {NAN} & {NAN} & {NAN} \\
FFC-AF     & \textbf{11.139} & \textbf{17.620} & \textbf{25.714} & \textbf{55.170} & \textbf{1109.128} \\
\bottomrule
\end{tabular}
\end{table}

\begin{table*}[t]
\centering
\caption{Comparison between FAC and AC Methods. FAC: full proposed method; AC: ablated version.}
\label{tab:fac_ac_results}
\scriptsize
\setlength{\tabcolsep}{2pt}
\renewcommand{\arraystretch}{1.0}
\begin{tabular}{@{}l*{20}{S[table-format=1.3]}@{}}
\toprule
\multirow{2}{*}{Method} & 
\multicolumn{4}{c}{Law School} & 
\multicolumn{4}{c}{Credit} & 
\multicolumn{4}{c}{Bank} & 
\multicolumn{4}{c}{Zafar} & 
\multicolumn{4}{c}{Census II} \\
\cmidrule(lr){2-5} \cmidrule(lr){6-9} \cmidrule(lr){10-13} \cmidrule(lr){14-17} \cmidrule(lr){18-21}
& {ACC} & {NMI} & {Balance} & {MNCE} &
{ACC} & {NMI} & {Balance} & {MNCE} &
{ACC} & {NMI} & {Balance} & {MNCE} &
{ACC} & {NMI} & {Balance} & {MNCE} &
{ACC} & {NMI} & {Balance} & {MNCE} \\
\midrule
AC  & 0.665 & 0.074 & 0.769 & 1.000 & 0.417 & 0.110 & 0.572 & 0.977 & 0.873 & 0.136 & 0.183 & 0.993 & 0.992 & 0.935 & 0.652 & 0.992 & 0.465 & 0.018 & 0.461 & 0.900 \\
FAC & \textbf{0.703} & \textbf{0.074} & 0.765 & 0.999 & \textbf{0.419} & \textbf{0.112} & \textbf{0.573} & \textbf{0.978} & \textbf{0.876} & \textbf{0.143} & \textbf{0.185} & \textbf{0.993} & \textbf{0.993} & \textbf{0.943} & \textbf{0.653} & \textbf{0.993} & 0.456 & \textbf{0.131} & \textbf{0.906} & \textbf{0.999} \\
\bottomrule
\end{tabular}
\end{table*}

\section{Experiments}

\subsection{Experiment Setup}
\subsubsection{Datasets}
We conduct experiments on five real-world and synthetic fair datasets, including Bank\cite{moro2012bank}, Credit Card\cite{yeh2009comparisons}, Zafar\cite{zafar2017fairness}, Law School\cite{le2022survey}, and Census II\cite{cohen2025fair}.These datasets are widely used in the field of fair clustering\cite{tonin2025accelerating,ziko2021variational}. Dataset specifications are detailed in Table~\ref{tab:dataset}.

\subsubsection{Implementation details}
We incorporate five state-of-the-art fair clustering algorithms into our framework and conduct comparative analysis against their original implementations: 
spFC~\cite{kleindessner2019guarantees}, 
VFC~\cite{ziko2021variational}, 
FFC~\cite{pan2023fairness}, 
FMSC~\cite{li2024one} (adapted from multi-view to single-view operation), 
and fairletFC~\cite{chierichetti2017fair}. 
Hyperparameters for both standalone executions and second-module executions within our framework were configured following recommendations in respective source publications. 
For the proposed AFCF framework, we perform grid search over two key hyperparameters: 
the anchor size $m$ is selected from $\{2r, 4r, \dots, 20r\}$ where $r$ denotes sensitive attribute cardinality, 
and the balance coefficient $\alpha$ is chosen from $\{0.0001, 0.01, 1, 100\}$. 
The number of clusters is fixed to the ground-truth class count across all datasets. More details can be found in the appendix. All experiments are conducted on a machine with an AMD Ryzen 7 4800H CPU (8 cores, 2.90 GHz), 32 GB of RAM, and integrated graphics.
\subsubsection{Metrics}
We evaluate clustering performance and fairness using four widely-adopted metrics, all of which follow the "higher-is-better" principle.Clustering performance is evaluated using Accuracy (ACC) and Normalized Mutual Information (NMI), 
while fairness is quantified through Balance and Minimal Normalized Conditional Entropy (MNCE, quantifying distributional consistency between clusters and global data)~\cite{li2024one}, with formal definitions provided below: 
\begin{equation}
    \text{MNCE} = 
    \frac{
        \min_{k \in [c]} \left( -\sum_{i=1}^{t} \rho_i^{(k)} \log \rho_i^{(k)} \right)
    }{
        -\sum_{i=1}^{t} \rho_i \log \rho_i
    } \in [0, 1]
    \label{eq:mnce}
\end{equation}

\subsection{Experimental Results}
Tables 3 and 4 show that integrating the Anchor-based Fair Clustering Framework (AFCF) consistently improves computational efficiency across five algorithms and datasets.

First, the clustering quality metrics (NMI, Balance, ACC, and MNCE) highlight the effectiveness of AFCF in enhancing fairness-aware clustering. For example, when VFC is embedded within the AFCF framework (VFC-AF), the model achieves a slight improvement in NMI across multiple datasets, such as Law School (0.063 to 0.065) and Bank (0.044 to 0.062), while also maintaining or improving accuracy. On the other hand, SpFC-AF shows significant improvements, such as moving NMI from near zero (0.005) to 0.074 on Law School, and most importantly, overcoming computational bottlenecks, completing tasks on datasets (like Bank) where SpFC failed to finish.

In terms of computational efficiency, embedding the algorithms within AFCF significantly reduces execution time. SpFC-AF, which previously failed to run due to time constraints, now completes tasks within time complexity that scales linearly with the data size, demonstrating the efficiency gains of embedding fairness constraints into our framework. FMSC-AF, which initially failed to handle large datasets, now successfully processes them within an acceptable time, with runtimes under 10 seconds for certain datasets like Credit. Notably, the FFC-AF method, which has the longest runtime among the methods tested, still benefits from AFCF's optimization, completing in 1109 seconds on Census II, a marked advantage compared to the 1494 seconds required by VFC alone.

\begin{figure}[!t]
  \centering
  \includegraphics[width=0.9\linewidth]{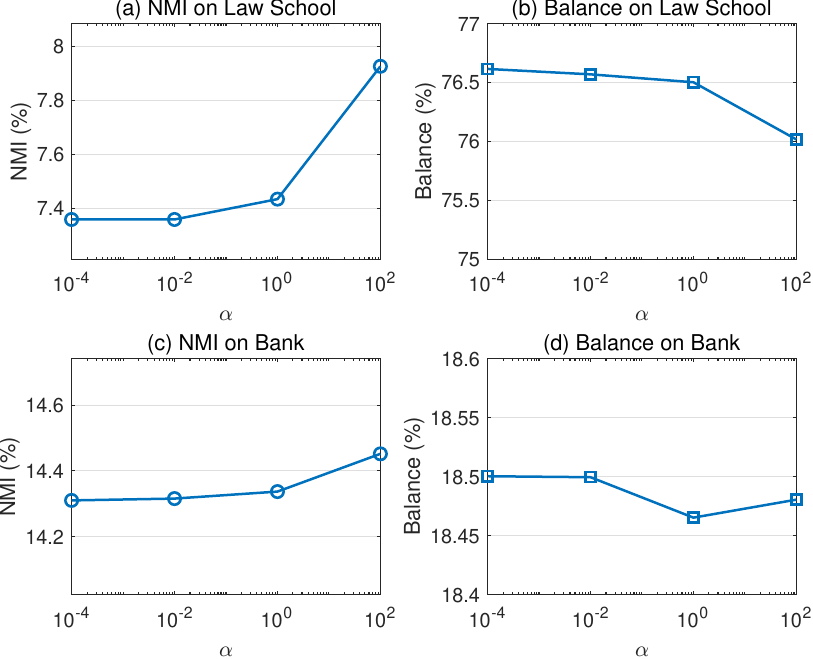}
  \caption{NMI and Balance on Law School and Bank data sets w.r.t.\ different values of $\alpha$.}
  \label{fig:sensitivity}
\end{figure}

\begin{figure}[t]
\centering
\includegraphics[width=0.9\linewidth]{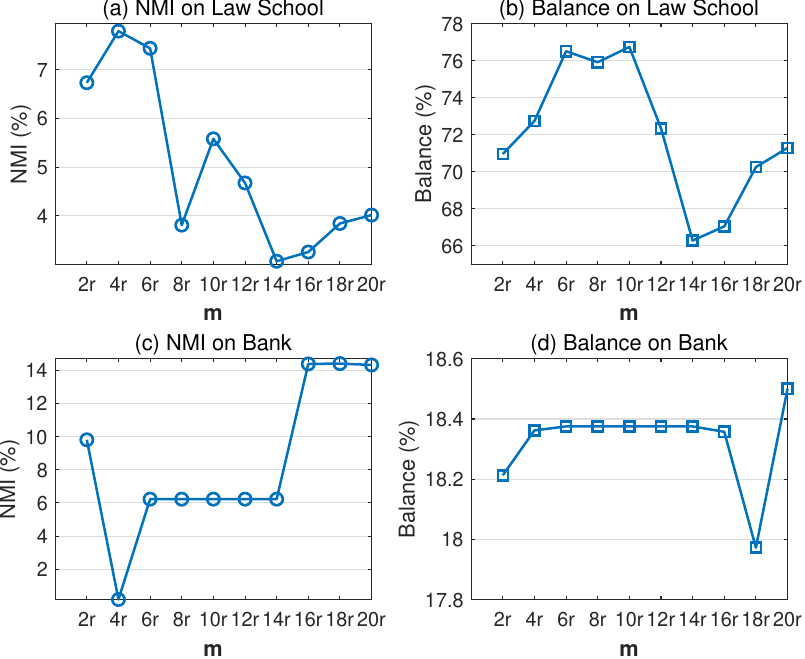}
\caption{NMI and Balance on Law School and Bank data sets w.r.t. different values of $m$.}
\label{fig:sensitivity_m}
\end{figure}

\subsection{Ablation Studies}
Table \ref{tab:ablation_results} demonstrates the critical importance of our fairness-aware anchor selection (FDAS) module. The complete clustering failure (NMI=0) observed with both Random and DAS methods on Law School highlights the necessity of fairness constraints in anchor selection for sensitive datasets. FDAS not only resolves this degenerate clustering issue but also achieves significant improvements across multiple benchmarks. On Bank data, FDAS substantially outperforms baselines in both clustering quality and fairness metrics. Most notably, for Census II, FDAS provides dramatic fairness improvements while maintaining competitive accuracy, confirming that FDAS is essential for preventing bias propagation in downstream clustering tasks.

Table \ref{tab:fac_ac_results} validates the necessity of our fairness-aware anchor graph construction (FAC) module. The results demonstrate that explicitly incorporating fairness constraints during graph construction consistently enhances both clustering quality and fairness metrics. Particularly on Census II, FAC significantly improves balance while maintaining competitive accuracy, confirming that fairness constraints prevent biased representation propagation in sensitive datasets. The observed improvements across multiple datasets highlight that fair graph construction is crucial for equitable clustering outcomes.

\subsection{Sensitivity and Convergence Analysis}
The robustness of our method is evaluated through systematic parameter variations. 
Figure~\ref{fig:sensitivity} shows NMI mildly increases with $\alpha$ ($10^{-4}$–$10^2$) while balance remains stable, confirming fairness preservation is $\alpha$-insensitive.

Figure.~\ref{fig:sensitivity_m}, Law School peaks at $m\!=\!4r$ (NMI) and $m\!=\!10r$ (balance), whereas Bank requires $m\!\geq\!16r$ for optimal NMI with consistently high balance. 
Fairness metrics thus exhibit stronger robustness to anchor selection than clustering quality.

Figure~\ref{fig:convergence} shows convergence behavior of Bank and Law School, which illustrates that the objective value of our algorithm consistently decreases with each iteration, which provides clear evidence of the convergence of our proposed algorithm.
\begin{figure}[!t]
  \centering
  \includegraphics[width=0.9\linewidth]{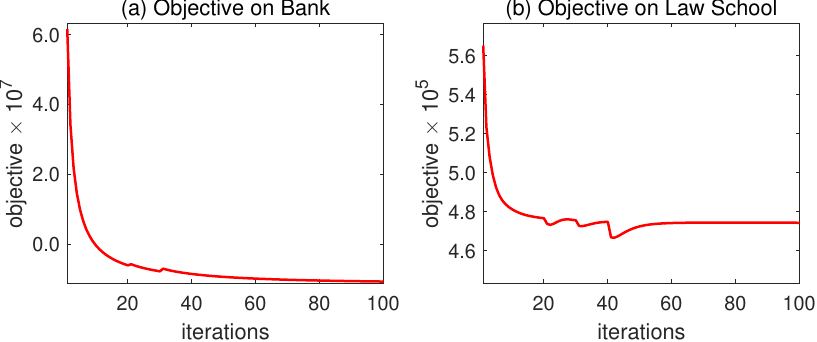}
  \caption{The convergence of the proposed algorithm for minimizing the objective in \eqref{eq:augmented_lag}. The plots are based on the Bank and Law School datasets.}
  \label{fig:convergence}
\end{figure}

\section{Conclusion}
This work addresses the critical scalability limitations in fair clustering through the Anchor-based Fair Clustering Framework (AFCF). By introducing a novel fair sampling strategy and fairness-preserving anchor graph construction with group-label joint constraints, AFCF enables linear-time complexity for arbitrary fair clustering algorithms. Our theoretical analysis guarantees global fairness equivalence with anchor-level clustering, while extensive experiments demonstrate orders-of-magnitude acceleration. The framework making large-scale fair clustering practically feasible.

\section{Acknowledgements}\label{acknowledgements}
This work is supported by the National Natural Science Foundation of China (No. 62376039 and 62506371).


\bibliography{aaai2026}

\appendix

\setcounter{secnumdepth}{0} 

%



\maketitle

\section{\LARGE\bfseries Appendix}

\begin{figure*}[!t]
  \centering
  \includegraphics[width=0.9\linewidth]{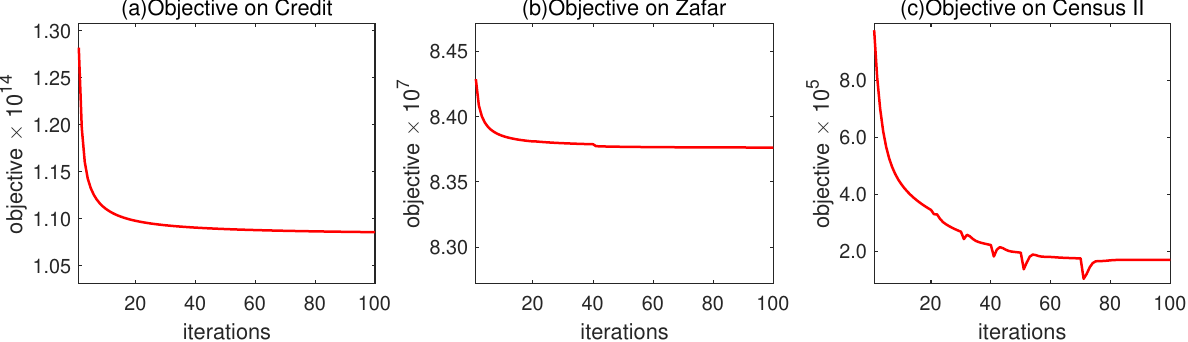}
  \caption{The convergence of the proposed algorithm for minimizing the objective. The plots are based on the Credit, Zafar and Census II datasets.}
  \label{fig:convergence}
\end{figure*}
\section{Proof of Proposition}
This section provides the formal proof demonstrating that the fairness-preserving optimization constraints maintain the balance metric invariant between the anchor clustering and final clustering.

\begin{proposition}\label{prop:fairness-preserving}
Under the fairness constraints, the balance metric $\text{balance}(\mathcal{C})$ of the final clustering equals that of the anchor clustering $\text{balance}(\mathcal{C}_{\text{a}})$:
\begin{equation}\label{eq:balance-equal}
\text{balance}(\mathcal{C}) = \text{balance}(\mathcal{C}_{\text{a}})
\end{equation}
\end{proposition}

\begin{proof}
The group-label joint fairness constraint ensures:
\begin{equation}
\sum_{j \in \mathcal{C}_l} \sum_{i \in G_r} \mathbf{Z}_{j,i} = |\mathcal{G}_{l,r}| \cdot \frac{n}{m}
\quad \forall l \in [c], r \in [t]
\end{equation}

This directly implies identical conditional distributions:
\begin{equation}
\rho_r^{(l)} = \frac{\sum_{j \in \mathcal{C}_l} \sum_{i \in G_r} \mathbf{Z}_{j,i}}{\sum_{j \in \mathcal{C}_l} \sum_{i=1}^n \mathbf{Z}_{j,i}} = 
\frac{|\mathcal{G}_{l,r}|}{|\mathcal{C}_l|} = \rho_r^{(l)}(\text{a})
\end{equation}

Therefore, the balance metric satisfies:
\begin{equation}\label{eq:balance-proof}
\begin{split}
\text{balance}(\mathcal{C}) 
&= \min_{l \in [c]} \min_{\substack{r, r' \in [t] \\ r \neq r'}} \frac{\rho_r^{(l)}}{\rho_{r'}^{(l)}} \\
&= \min_{l \in [c]} \min_{\substack{r, r' \in [t] \\ r \neq r'}} \frac{\rho_r^{(l)}(\text{a})}{\rho_{r'}^{(l)}(\text{a})} \\
&= \text{balance}(\mathcal{C}_{\text{a}})
\end{split}
\end{equation}
\end{proof}

\begin{figure}[!t]
  \centering
  \includegraphics[width=0.85\linewidth]{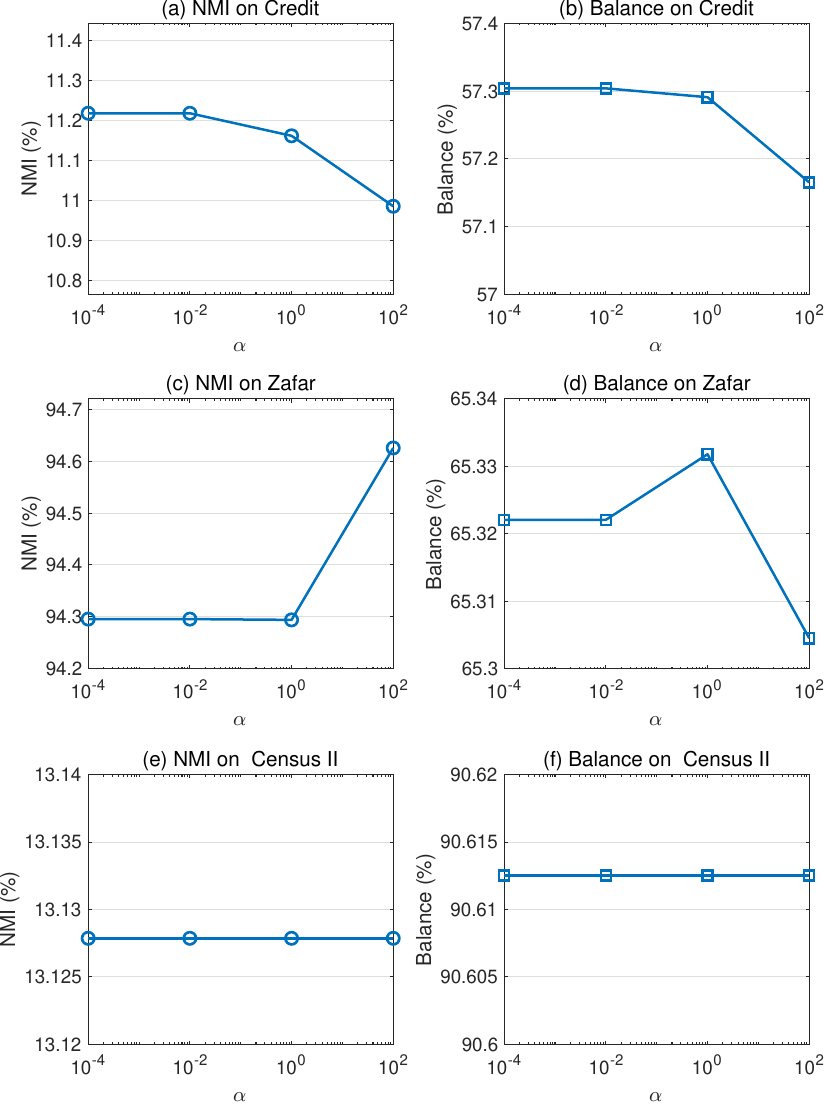}
  \caption{NMI and Balance on Credit, Zafar, and Census II data sets w.r.t.\ different values of $\alpha$.}
  \label{fig:sensitivity_new}
\end{figure}

\begin{figure}[t]
\centering
\includegraphics[width=0.98\linewidth]{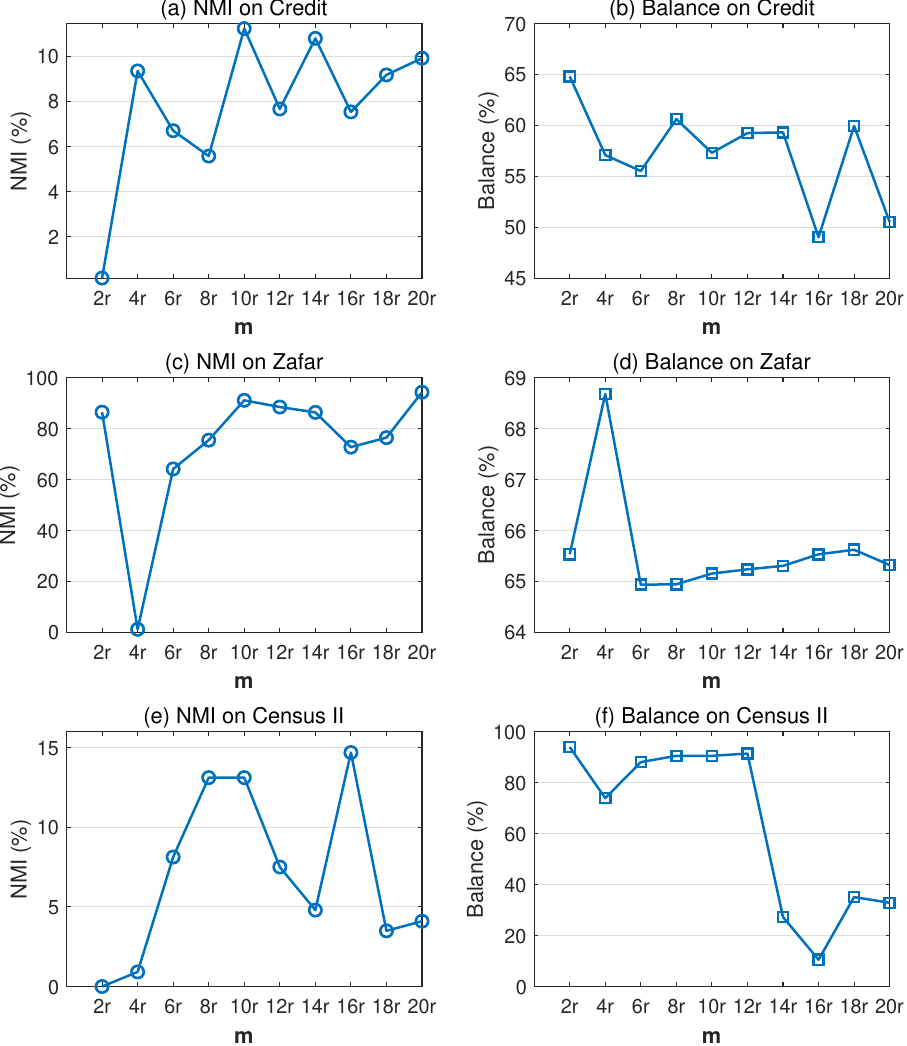} 
\caption{NMI and Balance on Credit, Zafar and Census II data sets w.r.t. different values of $m$.} 
\label{fig:sensitivity_m}
\end{figure}

\section{Detailed Method}

In this section, we formally present the Frank-Wolfe-based solution to the Z-subproblem (Algorithm~\ref{alg:fw}) and the details of our adaptive penalty adjustment scheme.

\subsubsection{Adaptive Penalty Adjustment}
To accelerate convergence, we implement an adaptive $\rho$ strategy. Every 10 iterations, we adjust $\rho$ based on primal ($r_k = \|\mathbf{Z}^k - \mathbf{E}^k\|_F$) and dual ($s_k = \rho \|\mathbf{E}^k - \mathbf{E}^{k-1}\|_F$) residuals:

\begin{equation}
\label{eq:rho_adaptive}
\rho^{k+1} = \begin{cases} 
\beta \rho^k & \text{if } r_k > \tau s_k \\
\rho^k / \beta & \text{if } s_k > \tau r_k \\
\rho^k & \text{otherwise}
\end{cases}
\end{equation}
with $\beta = 2$ and $\tau = 10$. This balancing ensures stable convergence while maintaining constraint satisfaction.

\begin{algorithm}
\caption{Frank-Wolfe Solver for Z-subproblem} 
\label{alg:fw} 
\begin{algorithmic}[1] 
\REQUIRE 
    Hessian matrix $\mathbf{Q} \in \mathbb{R}^{m \times m}$, 
    Gradient vector $\mathbf{c} \in \mathbb{R}^{m}$, 
    Simplex dimension $m$, 
    Maximum iterations $T$, 
    Initial solution $\mathbf{z_0} \in \Delta^m$
\ENSURE 
    Optimal solution $\mathbf{z} \in \Delta^m$
    
\STATE $\mathbf{z} \gets \mathbf{z_0}$ 
\FOR{$t = 1$ \textbf{to} $T$}
    \STATE $\mathbf{g} \gets \mathbf{Q}\mathbf{z} + \mathbf{c}$ 
    \STATE $j \gets argmin_{1 \leq l \leq m} \mathbf{g_l}$ 
    \STATE $\mathbf{s} \gets \mathbf{e_j}$ 
    \STATE $\mathbf{d} \gets \mathbf{s} - \mathbf{z}$ 
    \STATE $\delta \gets \mathbf{g}^\top \mathbf{d}$ 
    \IF{$\delta \geq -\epsilon_{fw}$} 
        \STATE \textbf{break} 
    \ENDIF 
    \STATE $q \gets \mathbf{d}^\top (\mathbf{Q} \mathbf{d})$ 
    \IF{$q \leq \epsilon_{curv}$}
        \STATE $\gamma \gets 1$ 
    \ELSE
        \STATE $\gamma \gets \min(1, \max(0, -\delta/q))$ 
    \ENDIF
    \STATE $\mathbf{z} \gets \mathbf{z} + \gamma \mathbf{d}$ 
\ENDFOR
\STATE \textbf{return} $\mathbf{z}$
\end{algorithmic} 
\end{algorithm}

\section{Details of Supplementary Experiments}
In this section, we present supplementary experiments on sensitivity analysis and convergence analysis. 
\subsection{sensitivity analysis}
Supplementary sensitivity analysis examines parameter robustness using Credit, Zafar, and Census II datasets. Figure~\ref{fig:sensitivity_new} shows NMI exhibits mild variations with increasing $\alpha$ while balance remains remarkably stable across all datasets, confirming fairness preservation is $\alpha$-insensitive. 
For anchor quantity sensitivity (Figure.~\ref{fig:sensitivity_m}), distinct patterns emerge across datasets: 
Credit achieves peak clustering performance at intermediate anchor volumes while maintaining acceptable balance preservation; 
Zafar exhibits stable balance invariance across configurations despite significant clustering quality variations, requiring larger anchor sets for optimal results; 
Census II demonstrates heightened sensitivity where both clustering quality and balance preservation necessitate careful anchor selection.
\subsection{convergence analysis}
Convergence behavior on supplementary datasets (Credit, Zafar, Census II) is shown in Figure~\ref{fig:convergence}, which illustrates that the objective value of our algorithm consistently decreases with each iteration, which provides clear evidence of the convergence of our proposed algorithm.


\end{document}